\documentclass[11pt,twoside]{article}
\usepackage[margin=1in]{geometry}

\usepackage[utf8]{inputenc} 
\usepackage[T1]{fontenc}    
\usepackage{hyperref}       
\usepackage{url}            
\usepackage{booktabs}       
\usepackage{amsmath,amssymb,amsfonts,amsthm,graphicx}       
\usepackage{nicefrac}       
\usepackage{microtype}      
\usepackage{wrapfig}
\usepackage{verbatim}
\usepackage{float}
\usepackage{enumerate}
\usepackage{enumitem}
\usepackage{multirow, array}
\usepackage{authblk}
\usepackage{algorithm}
\usepackage{algorithmic}

\usepackage{color}

\newtheorem{conj}{Conjecture}
\newtheorem{thm}{Theorem}
\newtheorem{rmk}[conj]{Remark}

\newtheorem{lem}[conj]{Lemma}
\newtheorem*{lem*}{Lemma}

\newtheorem{coro}[conj]{Corollary}

\newcommand{\supp}{\mathrm{supp}}

\newcommand{\avg}{\mathbb{E}}

\newcommand {\reals} {{\rm I\!R}}

\newcommand {\bv} {\mbox{\boldmath $v$}}

\newcommand {\bM} {\mbox{\boldmath $M$}}

\newcommand{\calB}{{\cal B}}

\newcommand{\calG}{{\cal G}}

\newcommand{\be}{\begin{equation}}
\newcommand{\ee}{\end{equation}}
\newcommand{\beqna}{\begin{eqnarray}}
\newcommand{\eeqna}{\end{eqnarray}}

\newcommand{\remove}[1]{}

\newfont{\bbb}{msbm10 scaled 500}

\newcommand{\vv}{{\bf v}}
\newcommand{\xv}{{\bf x}}

\newcommand{\zv}{{\bf z}}

\newcommand{\akshay}[1]{\textcolor{red}{[AK: #1]}}
\newcommand{\soumya}[1]{\textcolor{green}{[SP: #1]}}

\usepackage{hyperref}
\newcommand{\variation}[1]{\left\|#1\right\|_{\textrm{TV}}}

\title{Sample Complexity of Learning Mixtures  of \\ Sparse Linear Regressions}

\author{Akshay~Krishnamurthy\footnote{A.Krishnamurthy is with Microsoft Research at New York City, NY 10011 (e-mail:  \texttt{akshay@cs.umass.edu}).} ~~~Arya~Mazumdar\footnote{A. Mazumdar is with the Computer Science Department at the University of Massachusetts Amherst, Amherst, MA 01003, USA (email: \texttt{arya@cs.umass.edu}).}~~~Andrew~McGregor\footnote{A. McGregor is with the Computer Science Department at the University of Massachusetts Amherst, Amherst, MA 01003, USA (email: \texttt{mcgregor@cs.umass.edu}).} ~~~Soumyabrata~Pal\footnote{S. Pal is with the Computer Science Department at the University of Massachusetts Amherst, Amherst, MA 01003, USA (email: \texttt{spal@cs.umass.edu}).}}

\begin{document}

\maketitle

\begin{abstract}
In the problem of \emph{learning mixtures of linear regressions}, the goal is to learn a collection of signal vectors from a sequence of (possibly noisy) linear measurements, where each measurement is evaluated on an unknown signal drawn uniformly from this collection. This setting is quite expressive and has been studied both in terms of practical applications and for the sake of establishing theoretical guarantees. In this paper, we consider the case where the signal vectors are {\em sparse}; this generalizes the popular compressed sensing paradigm. 
We improve upon the state-of-the-art results as follows: In the noisy case, we resolve an open question of Yin et al.~(IEEE Transactions on Information Theory, 2019) by showing how to handle collections of more than two vectors and present the first robust reconstruction algorithm, i.e., if the signals are not perfectly sparse, we still learn a good sparse approximation of the signals. In the noiseless case, as well as in the noisy case, we show how to circumvent the need for a restrictive assumption required in the previous work.
Our techniques are  quite different from those in the previous work: for the noiseless case, we rely on a property of sparse polynomials and for the noisy case, we provide new connections to learning Gaussian mixtures and use ideas from the theory of error correcting codes.
\end{abstract}

\section{Introduction}
Learning mixtures of linear regressions is a natural generalization of the basic linear regression problem. In the basic problem, the goal is to learn the best linear relationship between the scalar responses (i.e., labels) and the explanatory variables (i.e., features). In the generalization, each scalar response is stochastically generated by picking a function uniformly from a set of $L$ unknown linear functions, evaluating this function on the explanatory variables and possibly adding noise; the goal is to learn the set of $L$ unknown linear functions.  
The problem was introduced by De Veaux \cite{de1989mixtures} over thirty years ago and has recently attracted growing interest~\cite{chaganty2013spectral,faria2010fitting,stadler2010l,viele2002modeling,yi2014alternating,yin2019learning}. Recent work focuses on a query-based scenario in which the input to the randomly chosen linear function can be specified by the learner. The {\em sparse} setting, in which each linear function depends on only a small number of variables, was recently considered by Yin et al.~\cite{yin2019learning}, and can be viewed as a generalization of the well-studied compressed sensing problem~\cite{candes2006robust,donoho2006compressed}. The problem has numerous applications in modelling heterogeneous data arising in medical applications, behavioral health, and music perception \cite{yin2019learning}.


\paragraph{Formal Problem Statement.} There are $L$ unknown {\color{black} distinct} vectors $\mathbf{\beta}^{1},\mathbf{\beta}^{2},\dots,\mathbf{\beta}^{L} \in \mathbb{R}^{n}$ and each is $k$-sparse, i.e., the number of non-zero entries in each $\beta^i$ is at most $k$ where $k$ is some known parameter.
We define an oracle $\mathcal{O}$ which, when queried with a vector $\mathbf{x} \in \mathbb{R}^{n}$, returns the noisy output $y \in \mathbb{R}$:
\begin{align}\label{eq:sample}
y= \langle \mathbf{x},\beta \rangle + \eta
\end{align}
where $\eta$ is a  random variable with $\avg \eta=0$ that represents the measurement noise and $\beta$ is chosen uniformly\footnote{Many of our results can be generalized to non-uniform distributions but we will assume a uniform distribution  throughout for the sake of clarity.} from the set  $\calB= \{
\mathbf{\beta}^{1},\mathbf{\beta}^{2},\dots,\mathbf{\beta}^{L}
\}$. 
The goal is to recover all vectors in $\calB$ by making a set of queries $\xv_1, \xv_2, \ldots, \xv_m$ to the oracle. We refer to the values returned by the oracle given these queries as \emph{samples.}
Note that the case of $L =1$  corresponds to the problem of  compressed sensing. Our primary focus is on the sample complexity of the problem, i.e., minimizing the number of queries that suffices to recover the sparse vectors up to some tolerable error.

\paragraph{Related Work.} The most relevant previous work is by Yin et al.~\cite{yin2019learning}.
For the noiseless case, i.e., $\eta=0$, they show that $O(kL \log (kL))$ queries are sufficient to recover all vectors in $\calB$ with high probability. However, their result requires a restrictive assumption on the set of vectors and do not hold for an arbitrary set of sparse vectors. Specifically, they require that for any $\beta, \beta'\in \calB$,  
\begin{equation}\label{eq:weirdassumption}
\beta_j \ne \beta'_j \quad \mbox{ for each } \quad j \in \supp(\beta) \cap \supp(\beta') \ .
\end{equation} 
Their approach depends crucially on this assumption and this limits the applicability of their approach. Note that our results will not depend on such an assumption.
For the noisy case, the approach taken by Yin et al.~only handles the $L =2$ case and they state the case of $L>2$ as an important open problem. Resolving this open problem will be  another one of our contributions. 

More generally, both compressed sensing \cite{candes2006robust,donoho2006compressed} and learning mixtures of distributions \cite{dasgupta1999learning,titterington1985statistical} are immensely popular topics across statistics, signal processing and machine learning with a large body of prior work. Mixture of linear regressions is a natural synthesis of mixture models and  linear regression, a very basic machine learning primitive~\cite{de1989mixtures}. Most of the work on the problem has considered learning generic vectors, i.e., not necessary sparse, and they propose a variety of algorithmic techniques to obtain polynomial sample complexity~\cite{chaganty2013spectral,faria2010fitting,kwon2018global,viele2002modeling,yi2016solving}. To the best of our knowledge, St\"adler et al.~\cite{stadler2010l} were the first to impose sparsity on the solutions. However, many of the earlier papers on mixtures of linear regression, essentially consider the queries to be fixed, i.e.,  part of the input, whereas in this paper, and in Yin et al.~\cite{yin2019learning}, we are interested in designing queries in such a way to minimize the number of queries. 


\paragraph{Our Results and Techniques.} We present results for both the noiseless and noisy cases. The latter is significantly more involved and is the main technical contribution of this paper.

 {\em Noiseless Case:}
In the case where there is no noise and the $L$ unknown vectors are $k$-sparse, we show that $O(kL\log(k L))$ queries suffice and that $\Omega(kL)$ queries are necessary. The upper bound matches the query complexity of the result by Yin et al.~but our result applies for all $k$-sparse vectors rather than just those satisfying the assumption in Eq.~\ref{eq:weirdassumption}. 
The approach we take is as follows:
In compressed sensing, exact recovery of $k$-sparse vectors is possible by taking samples with an $m \times n$ matrix with any $2k$ columns linearly independent. Such matrices exists with $m = 2k$ (such as Vandermonde matrices) and are called MDS matrices. We use rows of such a matrix repeatedly to generate samples. Since there are $L$ different vectors in the mixture, with $O(L \log L)$ measurements with a row we will be able to see the samples corresponding to each of the $L$ vectors with that row. However, even if this is true for measurements with each rows, we will still not be able to align measurements across the rows. For example, even though we will obtain $\langle \xv, \beta^\ell \rangle$ for all $\ell \in [L]$ and for all $\xv$ that are rows of an MDS matrix,  we will be unable to identify the samples corresponding to $\beta^1$. To tackle this problem, we propose using a special type of MDS matrix that allows us to align measurements corresponding to the same $\beta$s. After that, we just use the sparse recovery property of the MDS matrix to individually recover each of the vectors.


{\em Noisy Case:} We assume that the noise $\eta$ is a Gaussian random variable with zero mean. Going forward, we  write $\mathcal{N}(\mu,\sigma^2)$ to denote a Gaussian distribution with mean $\mu$ and variance $\sigma^2$. Furthermore, we will no longer assume vectors in $\calB$ are necessarily sparse. From the noisy samples, our objective is to recover an estimate $\hat{\beta}$ for each $\beta \in \calB$ such that
\begin{align}\label{eq:guarantee}
\|\beta - \hat{\beta}\| \le c\|\beta- \beta^*\|,
\end{align}
where $c$ is an absolute constant and  $\beta^{*}$ is the best $k$-sparse approximation of $\beta$, i.e., all except the largest (by absolute value) $k$ coordinates set to $0$. The norms in the above equation can be arbitrary defining the strength of the guarantee, e.g., when we refer to an $\ell_1/\ell_1$ guarantee both norms are $\|\cdot\|_1$. Our results should be contrasted with \cite{yin2019learning}, where results not only hold for only $L=2$ and under assumption~\eqref{eq:weirdassumption}, but the vectors are also strictly $k$-sparse. However, like \cite{yin2019learning}, we assume $\epsilon$-precision of the unknown vectors, i.e., the value in each coordinate of each $\beta\in \calB$ is an integer multiple of $\epsilon$.\footnote{Note that we do not assume $\epsilon$-precision in the noiseless case.}

{\color{black} Notice that in this model the noise is additive and not multiplicative. Hence, it is possible to increase the $\ell_2$ norm of the queries arbitrarily so that the noise becomes inconsequential. However, in a real setting, this cannot be allowed since increasing the strength (norm) of the queries has a cost and it is in our interest to minimize the cost. Suppose the algorithm designs the $i^{th}$ query vector by first choosing a distribution $Q_i$ and subsequently sampling a query vector $\mathbf{x}_i \sim Q_i$. Let us now define the signal to noise ratio as follows:
\begin{align}\label{eq:snr}
\mathsf{SNR}= \max_{i} \min_{\ell} \frac{\avg_{\mathbf{x}_i \sim Q_i} |\langle \mathbf{x}_i,\mathbf{\beta}^{\ell} \rangle|^{2}}{\avg \eta^2} \ .
\end{align}
Our objective in the noisy setting is to recover the unknown vectors $\beta^1,\beta^2,\dots,\beta^L \in \mathbb{R}^n$ while minimizing the number of queries and the $\mathsf{SNR}$ at the same time. In this setting, assuming that all the unknown vectors have unit norm, we show that $O(k\log^3 n \exp((\sigma/\epsilon)^{2/3}))$ queries with $\mathsf{SNR}=O(1/\sigma^2)$  suffice to reconstruct the $L=O(1)$ vectors in $\calB$ with the approximation guarantees given in Eq.~\eqref{eq:guarantee} with high probability if the noise $\eta$ is a zero mean gaussian with a variance of $\sigma^2$. This is equivalent to stating that $O(k\log^3 n \exp(1/(\epsilon\sqrt{\mathsf{SNR}})^{2/3}))$ queries suffice to recover the $L$ unknown vectors with high probability.
} \\
Note that in the previous work $\epsilon\sqrt{\mathsf{SNR}}$ is assumed to be at least constant and, if this is the case, our result is optimal up to polynomial factors since $\Omega(k)$ queries are required even if $L=1$. More generally, the dependence upon $\epsilon\sqrt{\mathsf{SNR}}$ in our result improves upon the dependence in the result by Yin et al. Note that we assumed $L=O(1)$ in our result because the dependence of sample complexity on $L$ is complicated as it is implicit in the signal-to-noise ratio.

As in noiseless case, our approach is to use a compressed sensing matrix and use its rows multiple time as queries to the oracle. At the first step, we would like to separate out the different $\beta$s from their samples with the same rows. Unlike the noiseless case, even this turns out to be a difficult task. Under the assumption of Gaussian noise, however, we are able to show that this is equivalent to learning a mixture of Gaussians with different means. In this case, the means of the Gaussians belong to an ``$\epsilon$-grid", because of the assumption on the precision of $\beta$s. This is not a standard setting in the literature of learning Gaussian mixtures, e.g.,~\cite{arora2001learning,hardt2015tight,moitra2010settling}. Note that, this is possible if the vector that we are sampling with has integer entries. As we will see a binary-valued compressed sensing matrix will do the job for us. We will rely on a novel complex-analytic technique to exactly learn the means of a mixture of Gaussians, with means belonging to an $\epsilon$-grid. This technique is 
paralleled by the recent developments in trace reconstructions where similar methods were used for learning a mixture of binomials~\cite{krishnamurthy2019trace,NazarovP17}. 

Once for each query, the samples are separated, we are still tasked with aligning them so that we know the samples produced by the same $\beta$ across different queries. The method for the noiseless case fails to work here. Instead, we use a new method motivated by error-correcting codes. In particular, we perform several redundant queries, that help us to do this alignment. For example, in addition to the pair of queries $\xv_i,\xv_j,$ we also perform the queries defined by $\xv_i+\xv_j$ and $\xv_i -\xv_j$.

 After the alignment, we use the compressed sensing recovery to estimate the unknown vectors. For this, we must start with a matrix that with minimal number of rows, will allow us to recover any vector with a guarantee such as \eqref{eq:guarantee}.
 On top of this, we also need the matrix to  have integer entries so that we can use our method of learning a mixture of Gaussians with means on an $\epsilon$-grid. Fortunately, a random binary $\pm 1$ matrix satisfies all the requirements~\cite{baraniuk2008simple}. 
 Putting now these three steps of learning mixtures, aligning and compressed sensing, lets us arrive at our results.

While we concentrate on sample complexity in this paper, our algorithm for the noiseless case is computationally efficient, 
and the only computationally inefficient step in the general noisy case is that of learning Gaussian mixtures. However, in practice one can perform a simple clustering (such as Lloyd's algorithm) to learn the means of the mixture. 

\paragraph{Organization and Notation.}
In Section \ref{sec:noiseless}, we present our results for the noiseless case. In Section \ref{sec:Lis2} we consider the case with noise when $L=2$ and then consider noise and general $L$ in Section \ref{sec:generalL}. Most proofs are deferred to the appendix in the supplementary material.
Throughout, we write $x\in_R X$ to denote taking an element $x$ from a finite set $X$ uniformly at random. For $n\in \mathbb{N}$, let $[n]:=\{1,2,\ldots, n\}$.


\section{Exact sparse vectors and noiseless samples}\label{sec:noiseless}

To begin, we deal with the case of uniform mixture of exact sparse vectors with the oracle returning noiseless answers when queried with a vector. For this case, our scheme is provided in Algorithm \ref{alg:noiseless}. The main result for this section is the following.
\begin{thm}\label{thm:noiseless}
 For a collection of $L$ vectors $\beta^{1},\beta^{2},\dots,\beta^{L} \in \mathbb{R}^{n}$ such that $\|\beta^i\|_{0} \le k \; \forall i \in [L]$, one can recover all of them exactly with probability at least $1-{3}/{k}$ with a total of  
$2kL\log Lk^2$ oracle queries. See Algorithm \ref{alg:noiseless}.
\end{thm}

\begin{algorithm}[tb]
\caption{\texttt{Noiseless Recovery} The algorithm for extracting recovering vectors via queries to oracle in noiseless setting. \label{alg:noiseless}}
\begin{algorithmic}[1]
\REQUIRE Number of unknown sparse vectors $L$, dimension $n$, sparsity $k$.
\STATE 
Let $t\in_R \{0,1,2, \ldots, k^2L^2-1\}$ and define $\alpha_1,\alpha_2,\ldots, \alpha_{2k}$ where $\alpha_j=\frac{2kt+j}{2k^3 L^2}$.
\FOR{$i=1,2,\dots,2k$}
\STATE  Make $ L\log (Lk^2)$ oracle queries with   vector $[1 \; \alpha_i \; \alpha_i^2 \; \dots \; \alpha_i^{n-1}]$. Refer to these as a \emph{batch}.
\ENDFOR
\FOR{$i=1,2,\dots,2k$}
\STATE For each batch of query responses corresponding to the same query vector, retain unique values and sort them in ascending order. Refer to this as the \textit{processed batch}.
\ENDFOR
\STATE Set matrix $\mathcal{Q}$ of dimension $2k \times L$ such that its $j^{th}$ row is the processed batch corresponding to the query vector $[1 \; \alpha_j \; \alpha_j^2 \; \dots \; \alpha_j^{n-1}]$
\FOR{$i=1,2,\dots,L$}
\STATE Decode the $i^{th}$ column of the matrix $\mathcal{Q}$ to recover $\beta^i$.  
\ENDFOR
\STATE Return $\beta^1,\beta^2,\dots,\beta^L$.
\end{algorithmic}
\end{algorithm}

A Vandermonde matrix is a matrix such that the entries in each row of the matrix are in geometric progression i.e., for an $m \times n$ dimensional Vandermonde matrix the entry in the $(i,j)$th entry is $\alpha_{i}^{j}$ where $\alpha_1,\alpha_2,\dots,\alpha_m \in \mathbb{R}$ are distinct  values. We will use the following useful property of the Vandermonde matrices; see, e.g.,  \cite[Section XIII.8]{gantmakher1959theory} for the proof.

\begin{lem}\label{lem:vmrank}
The rank of any $m \times m$ square submatrix of a Vandermonde matrix is $m$ assuming $\alpha_1,\alpha_2,\dots,\alpha_m$ are distinct and positive.
\end{lem}
This implies that, with the samples from a $2k \times n$ Vandermonde matrix, a $k$-sparse vector can be exactly recovered. This is because for any two unknown vectors $\beta$ and $\hat{\beta}$, the same set of responses for all the $2k$ rows of the Vandermonde matrix implies that  a $2k \times 2k$ square submatrix of the Vandermonde matrix is not full rank which is a contradiction to Lemma \ref{lem:vmrank}.

We are now ready to prove Theorem \ref{thm:noiseless}.

\begin{proof}
For the case of $L=1$, note that the setting is the same as the well-known compressed sensing problem. Furthermore, suppose a $2k \times n$ matrix has the property that any $2k \times 2k$ submatrix is full rank, then using the rows of this matrix as queries is sufficient to recover any $k$-sparse vector. By Lemma \ref{lem:vmrank}, any $2k\times n$ Vandemonde matrix has the necessary property.

Let $\beta^{1},\beta^{2},\dots,\beta^{L}$ be the set of unknown $k$-sparse vectors. Notice that a particular row of the Vandermonde matrix looks like $[1 \; z \; z^{2} \; z^{3} \; \dots \; z^{n-1}]$ for some value of $z \in \mathbb{R}$. Therefore, for some vector $\beta^{i}$ and a particular row of the Vandermonde matrix, the inner product of the two can be interpreted as a degree $n$ polynomial evaluated at $z$ such that the coefficients of the polynomial form the vector $\beta^{i}$. More formally, the inner product can be written as 
$
f^{i}(z)=\sum_{j=0}^{n-1} \beta^{i}_j z^{j}
$
where $f^{i}$ is the polynomial corresponding to the vector $\beta^i$. For any value $z \in \mathbb{R}^n$, we can define an ordering over the $L$ polynomials $f^1,f^2,\dots,f^L$ such that $f^i > f^j$ iff $f^i(z)>f^j(z)$.

For two distinct indices $i,j \in [L]$, we will call the polynomial $f^i-f^j$ a \textit{difference polynomial}. 
Each difference polynomial has at most $2k$ non-zero coefficients and therefore has at most $2k$ positive roots by Descartes' Rule of Signs \cite{curtiss1918recent}. Since there are at most ${L(L-1)}/{2}$ distinct difference polynomials, the total number of distinct values that are roots of at least one difference polynomial is less than $kL^{2}$. Note that if an interval does not include any of these roots, then the ordering of $f^1, \ldots, f^L$ remains consistent for any point in that interval. In particular, consider the intervals $(0,\gamma], (\gamma, 2\gamma], \ldots, (1-\gamma, 1]$ where $\gamma=1/(k^2L^2)$. At most $kL^{2}$ of these intervals include a root of a difference polynomial and hence if we  pick a random interval then with probability at least  $1-1/k$, the ordering of $f^1, \ldots, f^L$ are consistent throughout the interval. If the interval chosen is $(t\gamma,(t+1)\gamma]$ then set $\alpha_j=t\gamma+j\gamma/(2k)$ for $j=1,\ldots, 2k$.


Now for each value of $\alpha_i$, define the vector $\mathbf{x}_i \equiv [1 \; \alpha_i \; \alpha_i^{2} \; \alpha_i^{3} \; \dots \; \alpha_i^{n-1}]$.
For each $i\in [2k]$, the vector $\mathbf{x}_i$ will be used as query to the oracle repeatedly for $L\log Lk^2$ times. We will call the set of query responses from the oracle for a fixed query vector $\mathbf{x}_i$ a \textit{batch}. For a fixed batch and $\beta^j$,
\begin{align*}
\Pr(\beta^j \textup{ is not sampled by the oracle in the batch}) \le \Big(1-\frac{1}{L}\Big)^{L\log Lk^2} \le e^{-\log Lk^2}=\frac{1}{Lk^2}.
\end{align*}
Taking a union bound over all the vectors ($L$ of them) and all the batches ($2k$ of them), we get that in every batch every vector $\beta^j$ for $j \in [L]$ is sampled with probability at least $1-{2}/{k}$. Now, for each batch, we will retain the unique values (there should be exactly $L$ of them with high probability) and sort the values in each batch. Since the ordering of the polynomial remains same, after sorting, all the values in a particular position in each batch correspond to the same vector $\beta^j$ for some unknown index $j\in [L]$. We can aggregate the query responses of all the batches in each position and since there are $2k$ linear measurements corresponding to the same vector, we can recover all the unknown vectors $\beta^{j}$ using Lemma \ref{lem:vmrank}.   The failure probability of this algorithm is at most ${3}/{k}$.
\end{proof}

The following theorem establishes that our method is almost optimal in terms of sample complexity.

\begin{thm}\label{thm:noiselesslb}
At least $2Lk$ oracle queries are necessary to recover an arbitrary set of  $L$ vectors that are $k$-sparse. 
\end{thm}

\section{Noisy Samples and Sparse Approximation}
We now consider the more general setting where the oracle is noisy and the vectors $\beta^1, \dots, \beta^L$ are not necessarily sparse. We assume $L$ is an arbitrary constant, i.e.,  it does not grow with $n$ or $k$ and that the unknown vectors have $\epsilon$ precision, i.e., each entries is an integer multiple of $\epsilon$.
The noise will be Gaussian with zero mean and variance $\sigma^2,$ i.e., $\eta \sim \mathcal{N}(0,\sigma^2)$. Our main result of this section is the following.
{
\color{black}
 \begin{thm}\label{thm:bigone}
  It is possible to recover approximations with the  $\ell_1/\ell_1$ guarantee in Eq.~\eqref{eq:guarantee} with probability at least $1-{2}/{n}$ of all the unknown vectors $\beta^\ell \in \{0, \pm \epsilon, \pm 2\epsilon, \pm 3\epsilon, \dots\}^n, \ell =1, \dots, L$ with
  $O(k (\log^3 n) \exp ((\sigma/\epsilon)^{2/3})$ oracle queries where $\mathsf{SNR}=O(1/\sigma^2)$.
\end{thm}

}
{
Before we proceed with the ideas of proof, it would be useful to recall the {\em restricted isometry property} (RIP) of matrices in the context of recovery guarantees of ~\eqref{eq:guarantee}.   A matrix $\Phi \in \reals^{m \times n}$ satisfies the $(k,\delta)$-RIP if for any vector $\zv \in \reals^n$ with $\|\zv\|_0 \le k,$
 \begin{align}\label{eq:rip}
(1-\delta)\|\zv\|_2^2 \le \|\Phi\zv\|_2^2 \le (1+\delta)\|\zv\|_2^2.       
 \end{align} 
 It is known that if a matrix is $(2k, \delta)$-RIP with $\delta < \sqrt{2} -1$, then the guarantee of ~\eqref{eq:guarantee} (in particular, $\ell_1/\ell_1$-guarantee and also an $\ell_2/\ell_1$-guarantee) is possible~\cite{candes2008restricted} with the {\em the basis pursuit} algorithm, an efficient algorithm based on linear programming. It is also known that a random $\pm 1$ matrix (with normalized columns) satisfies the property with $c_s k \log n$ rows, where $c_s$ is an absolute constant ~\cite{baraniuk2008simple}.

%

There are several key ideas of the proof. Since the case of $L=2$ is simpler to handle, we start with that and then provide the extra steps necessary for the general case subsequently.

\subsection{Gaussian Noise: Two vectors}\label{sec:Lis2}
\begin{algorithm}[tb]
\caption{\texttt{Noisy Recovery for $L=2$} The algorithm for recovering best $k$-sparse approximation of vectors via queries to oracle in noisy setting. \label{alg:noisy}}
\begin{algorithmic}[1]
\REQUIRE  $\mathsf{SNR}={1}/{\sigma^2}$, Precision of unknown vectors $\epsilon$, and the constant $c_s$ where $c_sk \log n$ rows are sufficient for RIP in binary matrices.
\FOR{$i=1,2,\dots,c_s k\log ({n}/{k})$} 
\STATE Call SampleAndRecover($\mathbf{v}_i$) where $\mathbf{v}_i \in_R \{+1,-1\}^{n}$.
\ENDFOR
\FOR{$i \in [\log n]$ and $j\in [c_s k\log ({n}/{k})]$ with $j\neq i$}
\STATE Call SampleAndRecover($(\mathbf{v}_i+\mathbf{v}_j)/2$) and SampleAndRecover($(\mathbf{v}_i-\mathbf{v}_j)/2$)
\ENDFOR
\STATE Choose vector $\mathbf{v}$ from $\{\mathbf{v}_1,\mathbf{v}_2,\dots,\mathbf{v}_{\log n}\}$ such that $\langle \mathbf{v},\beta^1 \rangle \neq \langle \mathbf{v},\beta^2 \rangle$. 
\FOR{$i=1,2,\dots,k\log ({n}/{k})$ and $\mathbf{v}_i \neq \mathbf{v}$}
\STATE Label one of $\langle \mathbf{v}_i,\beta^1 \rangle$,$\langle \mathbf{v}_i,\beta^2  \rangle$  to be $\langle \mathbf{v},\beta^1 \rangle$ if their sum is in the pair 
$\langle \frac{\mathbf{v}_i+\mathbf{v}}{2},\beta^1 \rangle$,$\langle \frac{\mathbf{v}_i+\mathbf{v}}{2},\beta^2  \rangle$ and their difference is in the pair $\langle \frac{\mathbf{v}-\mathbf{v}_i}{2},\beta^1 \rangle$,$\langle \frac{\mathbf{v}-\mathbf{v}_i}{2},\beta^2  \rangle$. Label the other $\langle v,\beta^2 \rangle$.
\ENDFOR
\STATE Aggregate all (query, denoised query response pairs) labelled $\langle \mathbf{v},\beta^1 \rangle$ and $\langle \mathbf{v},\beta^2 \rangle$ separately and multiply all denoised query responses by a factor of $1/(\sqrt{c_sk\log (n/k)})$. 
\STATE Return best $k$-sparse approximation of $\beta^1$ and $\beta^2$ by using Basis Pursuit algorithm on each aggregated cluster of (query, denoised query response) pairs.
\[\] \vspace{-1cm}
\STATE \textbf{function} SampleAndRecover (\bv)
\STATE \ \ \ \ Issue $T = c_2 \exp\left((\sigma/\epsilon)^{2/3}\right)$ queries to oracle with $\bv$.
\STATE \ \ \ \ Return $\langle\bv,\beta^1\rangle$, $\langle\bv,\beta^2\rangle$ via min-distance estimator (Gaussian mixture learning, lemma~\ref{ref:gmm}).  
\STATE \textbf{end function} 

\end{algorithmic}
\end{algorithm}

Algorithm~\ref{alg:noisy} addresses the setting with only two unknown vectors.
We will assume $\|\beta^1\|_2 = \|\beta^2\|_2 = 1$, so that we can subsequently show that the SNR is simply $1/\sigma^2$. {\color{black} This assumption is not necessary but we make this for the ease of presentation.} The assumption of $\epsilon$-precision for $\beta$ was made in Yin et al.~\cite{yin2019learning}, and we stick to the same assumption. On the other hand, Yin et al.~requires further assumptions that we do not need to make. Furthermore, the result of Yin et al.~is restricted to exactly sparse vectors, whereas our result holds for general sparse approximation.


For the two-vector case the result we aim to show is following.
\begin{thm}\label{thm:noisy}
Algorithm \ref{alg:noisy} uses $O(k \log^{3} n\exp ( ({\sigma}/{\epsilon})^{2/3} ))$ queries to recover both the vectors $\beta^1$ and $\beta^2$ with an $\ell_1/\ell_1$ guarantee in Eq.~\eqref{eq:guarantee} with probability at least $1-{2}/{n}$.
\end{thm}
This result is directly comparable with~\cite{yin2019learning}. On the statistical side, we improve their result in several ways: (1) we improve the dependence on $\sigma/\epsilon$ in the sample complexity from $\exp(\sigma/\epsilon)$ to $\exp( (\sigma/\epsilon)^{2/3} )$,\footnote{Note that~\cite{yin2019learning} treat $\sigma/\epsilon$ as constant in their theorem statement, but the dependence can be extracted from their proof.} (2) our result applies for dense vectors, recovering the best $k$-sparse approximations, and (3) we do not need the overlap assumption (eq.~\eqref{eq:weirdassumption}) used in their work. 


Once we show $\mathsf{SNR}={1}/{\sigma^2}$,  Theorem~\ref{thm:noisy} trivially implies  Theorem~\ref{thm:bigone} in the case $L=2$. Indeed,
from Algorithm \ref{alg:noisy}, notice that we have used  vectors  $\vv$ sampled uniformly at random from $\{+1,-1\}^n$ and use them as query vectors. We must have
$
    \avg_{\mathbf{v}}|\langle\mathbf{v},\beta^\ell\rangle|^2/\avg \eta^2=\|\beta^\ell\|_2^2/\sigma^2={1}/{\sigma^2}
$
for $\ell=1,2$. Further, we have used the sum and difference query vectors which have the form $(\mathbf{v_1+v_2})/{2}$ and $(\mathbf{v_1-v_2})/{2}$ respectively where $v_1,v_2$ are sampled uniformly and independently from $\{+1,-1\}^n$. Therefore, we must have for $\ell=1,2$, 
$
    \avg_{\mathbf{v_1},\mathbf{v_2}} |\langle (\mathbf{v_1 \pm v_2})/2,\beta^\ell \rangle|^2/\avg \eta^2={1}/{2\sigma^2}.
$
According to our definition of $\mathsf{SNR}$, we  have that $\mathsf{SNR}={1}/{\sigma^2}$.

A description of Algorithm  \ref{alg:noisy} that lead to proof of Theorem \ref{thm:noisy} can be found in Appendix~\ref{app:noisy_two}. We provide a short sketch here and state an important lemma that we will use in the more general case. 

The main insight is that for a fixed sensing vector $\bv$, if we repeatedly query with $\bv$, we obtain samples from a mixture of Gaussians $\tfrac{1}{2}\mathcal{N}(\langle \bv, \beta^1\rangle, \sigma^2) + \tfrac{1}{2}\mathcal{N}(\langle \bv, \beta^2\rangle, \sigma^2)$. If we can \emph{exactly} recover the means of these Gaussians, we essentially reduce to the noiseless case from the previous section. 
The first key step upper bounds the sample complexity for exactly learning the parameters of a mixture of Gaussians.
\begin{lem}[Learning Gaussian mixtures]\label{ref:gmm}
Let $\mathcal{M}=\frac{1}{L}\sum_{i=1}^L\mathcal{N}(\mu_i,\sigma^2)$
be a uniform mixture of $L$ univariate Gaussians, with known shared
variance $\sigma^2$ and with means $\mu_i \in
\epsilon\mathbb{Z}$. Then, for some constant $c>0$ and some $t=\omega(L)$, there exists an algorithm that requires $ ctL^2 \exp((\sigma/\epsilon)^{2/3})$ samples from $\mathcal{M}$ and
exactly identifies the parameters $\{\mu_i\}_{i=1}^L$ with
probability at least $1-2e^{-2t}$.
\end{lem}

If we sense with $\bv \in \{-1,+1\}^n$ then $\langle\bv,\beta^1\rangle,\langle\bv,\beta^2\rangle \in \epsilon \mathbb{Z}$, so appealing to the above lemma, we can proceed assuming we know these two values exactly. 
Unfortunately, the sensing vectors here are more restricted --- we must maintain bounded SNR and our technique of mixture learning requires that the means have finite precision --- so we cannot simply appeal to our noiseless results for the alignment step.
Instead we design a new alignment strategy, inspired by error correcting codes. Given two query vectors $\bv_1,\bv_2$ and the exact means $\langle \bv_i,\beta^j\rangle$, $i,j \in \{1,2\}$, we must identify which values correspond to $\beta^1$ and $\beta^2$. 
In addition to sensing with any pair $\bv_1$ and $\bv_2$ we sense with $\tfrac{\bv_1 \pm \bv_2}{2}$, and we use these two additional measurements to identify which recovered means correspond to $\beta^1$ and which correspond to $\beta^2$.  Intuitively, we can check if our alignment is correct via these reference measurements.


Therefore, we can obtain aligned, denoised inner products with each of the two parameter vectors. At this point we can apply a standard compressed sensing result as mentioned at the start of this section to obtain the sparse approximations of vectors. 

\subsection{General value of $L$}\label{sec:generalL}
In this setting, we will have $L>2$ unknown vectors $\beta^1,\beta^2,\dots,\beta^L \in \mathbb{R}^n$ of unit norm each from which the oracle can sample from with equal probability. We assume that $L$ does not grow with $n$ or $k$ and as before, all the elements in the unknown vectors lie on a $\epsilon$-grid. 
Here, we will build on the ideas for the special case of $L=2$. 

\begin{algorithm}[tbh]
\caption{\texttt{Noisy Recovery for any constant $L$} The algorithm for recovering best $k$-sparse approximation of vectors via queries to oracle in noisy setting.  \label{alg:noisygen}}
\begin{algorithmic}[1] 
\REQUIRE  $c',\alpha^{\star},z^{\star}$ as defined in equations \eqref{eq:c}, \eqref{eq:alpha} and \eqref{eq:z} respectively, Variance of noise $\avg \eta^2=\sigma^2$ and precision of unknown vectors as  $\epsilon$.
\FOR{$i=1,2,\dots,\sqrt{\alpha^{\star}}\log n+c'\alpha^{\star}k\log (n/k)$} 
\STATE Let $\mathbf{v}_i\in_R \{+1,-1\}^{n}$,  $\mathbf{r}_i\in_R \{-2z^{\star},-2z^{\star}+1,\dots, 2z^{\star}\}^{n}$, $q_i\in_R \{1,2,\dots,4z^{\star}+1\}$
\STATE 
Make $c_2\exp(({\sigma}/{\epsilon})^{2/3})$ queries to the oracle using each of the  vectors $(q_i-1)\mathbf{r_i},\mathbf{v_i}+q_i\mathbf{r_i}$ and  $\mathbf{v_i}+\mathbf{r_i}$.
\STATE Recover $\langle \{(q_i-1)\mathbf{r}_i,\beta^t \rangle\}_{t=1}^{L}$,$\{\langle \mathbf{v_i+r_i},\beta^t \rangle\}_{t=1}^{L}$,$\{\langle \mathbf{v}_i+q_i\mathbf{r}_i,\beta^t \rangle\}_{t=1}^{L}$ by using min-distance estimator (Gaussian mixture learning, lemma~\ref{ref:gmm}).
\ENDFOR

\FOR{$i\in [\sqrt{\alpha^{\star}}\log n]$ and $j\in [\alpha^{\star}k\log ({n}{k})]$}
\STATE Make $c_2\exp(({\sigma}/{\epsilon})^{2/3})$ queries to the oracle using the vector  $\mathbf{r}_{i+j}+\mathbf{r}_i$.
\STATE Recover $\{\langle \mathbf{r}_{i+j}+\mathbf{r}_i,\beta^t \rangle\}_{t=1}^{L}$, by using the min-distance estimator (Gaussian mixture learning, Lemma~\ref{ref:gmm}).
\ENDFOR
\STATE Choose vector $(\mathbf{v}^{\star},\mathbf{r}^{\star},q^{\star})$ from $\{(\mathbf{v}_t,\mathbf{r}_t,q_t)\}_{t=1}^{\sqrt{\alpha^{\star}}\log n}$ such that
$(\mathbf{v}^{\star}+\mathbf{r}^{\star},(q-1)\mathbf{r}^{\star},\mathbf{v}^{\star}+q^{\star}\mathbf{r}^{\star})$ is good. 
{
\color{black}
Call a triplet $(\mathbf{v}+\mathbf{r},(q-1)\mathbf{r}, \mathbf{v}+q\mathbf{r})$ to be good if no element in $\{\langle \mathbf{v}+q\mathbf{r},\beta^i \rangle\}_{i=1}^L$ can be written in two possible ways as sum of two elements, one each from $\{\langle \mathbf{v},\beta^i \rangle\}_{i=1}^L$ and $\{\langle \mathbf{v}+(q-1)\mathbf{r},\beta^i \rangle\}_{i=1}^L$.
}
\STATE Initialize $\mathcal{S}_j=\phi$ for $j=1,\dots,L$
\FOR{$i=\sqrt{\alpha^{\star}}\log n+1,2,\dots,\sqrt{\alpha^{\star}}\log n+c'\alpha^{\star}k\log \frac{n}{k}$}
\IF{$(\mathbf{v}_{i}+\mathbf{r}_i,(q_i-1)\mathbf{r}_{i},\mathbf{v}_{i}+q\mathbf{r}_i)$ is matching good with respect to $(\mathbf{v}^{\star}+\mathbf{r}^{\star},(q-1)\mathbf{r}^{\star},\mathbf{v}^{\star}+q^{\star}\mathbf{r}^{\star})$ 
({\color{black}
Call a triplet $(\mathbf{v}'+\mathbf{r}',(q'-1)\mathbf{r}', \mathbf{v}'+q'\mathbf{r}')$ to be matching good w.r.t a good triplet $(\mathbf{v}^{\star}+\mathbf{r}^{\star},(q^{\star}-1)\mathbf{r}^{\star}, \mathbf{v}^{\star}+q^{\star}\mathbf{r}^{\star})$ if $(\mathbf{v}'+\mathbf{r}',(q'-1)\mathbf{r}', \mathbf{v}'+q'\mathbf{r}')$ and $(\mathbf{r}',\mathbf{r}^{\star},\mathbf{r}'+\mathbf{r}^{\star})$ are good.
})}
\STATE Label the elements in $\{\langle \mathbf{v}_i,\beta^t \rangle\}_{t=1}^{L}$ as described in Lemma \ref{lem:lab}
\FOR{$j=1,2,\dots,L$}
\STATE $\mathcal{S}_j=\mathcal{S}_j\cup\{\langle \mathbf{v}_i,\beta^t \rangle\} $ if label of $\langle \mathbf{v}_{i},\beta^t \rangle$ is $\langle \mathbf{r}^{\star},\beta^j \rangle$
\ENDFOR
\ENDIF
\ENDFOR
\FOR{$j=1,2,\dots,L$}
\STATE Aggregate the elements of $\mathcal{S}_j$ and scale them by a factor of $1/c'k\log (n/k)$. 
\STATE Recover the vector $\beta^j$ by using basis pursuit algorithms (compressed sensing decoding).
\ENDFOR
\STATE Return $\beta^1,\beta^2,\dots,\beta^L$.
\end{algorithmic}
\end{algorithm}
 
The main result of this section is the following.
\begin{thm}\label{thm:noisygen}
Algorithm \ref{alg:noisygen} uses $O\Big(k (\log n)^{3}\exp \Big( (\frac{\sigma}{\epsilon})^{2/3} \Big)\Big)$ queries {\color{black}with $\mathsf{SNR}=O(1/\sigma^2)$} to recover all the vectors $\beta^1,\dots,\beta^L$ with  $\ell_1/\ell_1$ guarantees in Eq.~\eqref{eq:guarantee} with probability at least $1-{2}/{n}$.
\end{thm}
Theorem~\ref{thm:bigone} follows as a corollary of this result.

The analysis of Algorithm \ref{alg:noisygen} and the proofs of Theorems ~\ref{thm:bigone} and ~\ref{thm:noisygen} are provided in detail in Appendix~\ref{sec:ngenL}. Below we sketch some of the main points of the proof.

There are two main hurdles in extending the steps explained for $L=2$. For a query vector $\mathbf{v}$, we define the \textit{denoised query means} to be the set of elements $\{\langle \mathbf{v},\beta^i\rangle\}_{i=1}^{L}$. Recall that a query vector $\mathbf{v}$ is defined to be \textit{good} if all the elements in the set of denoised query means $\{\langle \mathbf{v},\beta^1 \rangle, \langle \mathbf{v},\beta^2 \rangle, \dots, \langle \mathbf{v},\beta^L \rangle\}$ are distinct. For $L=2$, the probability of a query vector $\mathbf{v}$ being \textit{good} for $L=2$ is at least $1/2$ but for a value of $L$ larger than $2$, it is not possible to obtain such guarantees without further assumptions. For a more concrete example, consider $L \ge 4$ and the unknown vectors $\beta^1,\beta^2,\dots,\beta^L$ to be such that $\beta^i$ has \texttt{1} in the $i^{th}$ position and zero everywhere else. If $\mathbf{v}$ is sampled from $\{+1,-1\}^n$ as before, then $\langle \mathbf{v},\beta^i \rangle$ can take values only in $\{-1,0,+1\}$ and therefore it is not possible that all the values $\langle \mathbf{v},\beta^i \rangle$ are distinct. Secondly, even if we have a \textit{good} query vector, it is no longer trivial to extend the clustering or alignment step. Hence a number of new ideas are necessary to solve the problem for any general value of $L$.

We need to define  a few constants which are used in the algorithm. Let $\delta < \sqrt{2}-1$ be a constant (we need a $\delta$ that allow $k$-sparse approximation given a $(2k,\delta)$-RIP matrix). Let $c'$ be a large positive constant such that 
\begin{align*}
    \frac{\delta^2}{16}-\frac{\delta^3}{48}-\frac{1}{c'} >0 \tag{A}\label{eq:c}.
\end{align*}
Secondly, let $\alpha^{\star}$ be another positive constant that satisfies the following for a given value of $c'$,
\begin{align*}
    \alpha^{\star}=\max \Big\{\alpha:\frac{\alpha^\alpha}{(\alpha-1)^{\alpha-1}} < \exp\Big(\frac{\delta^2}{16}-\frac{\delta^3}{48}-\frac{1}{c'}\Big)\Big\}. \tag{B}\label{eq:alpha}
\end{align*}
Finally, for a given value of $\alpha^{\star}$ and $L$, let $z^{\star}$
be the smallest integer that satisfies the following:
\begin{align*}
    z^{\star}=\min\Big\{z \in \mathbb{Z}:1-L^3\Big(\frac{3}{4z+1}-\frac{1}{{4z^{2}}+1}\Big) \ge \frac{1}{\sqrt{\alpha^{\star}}} \Big\}. \tag{C}\label{eq:z}
\end{align*}
\paragraph{The Denoising Step.} 
In each step of the algorithm, we  sample a vector $\mathbf{v}$ uniformly at random from $\{+1,-1\}^n$, another vector $\mathbf{r}$ uniformly at random from $\calG \equiv \{-2z^{\star},-2z^{\star}+1,\dots,2z^{\star}-1,2z^{\star}\}^n$ and a number $q$ uniformly at random from $\{1,2,\dots,4z^{\star}+1\}$. Now, we will use a batch of queries corresponding to the vectors $\mathbf{v+r},(q-1)\mathbf{r}$ and $\mathbf{v}+q\mathbf{r}$. 
 We  define a triplet of query vectors $(\mathbf{v}_1,\mathbf{v}_2,\mathbf{v}_3)$ to be \textit{good} if for all triplets of indices $i,j,k \in [L]$ such that $i,j,k$ are not identical,
\begin{align*}
    \langle \mathbf{v}_1,\beta^i \rangle+\langle \mathbf{v}_2,\beta^j \rangle \neq \langle \mathbf{v}_3,\beta^k \rangle.
\end{align*}
We show that the query vector triplet $(\mathbf{v+r},(q-1)\mathbf{r},\mathbf{v}+q\mathbf{r})$ is good with  at least some probability. 
This implies if we choose $O(\log n)$ triplets of such query vectors, then at least one of the triplets are good with probability $1-1/n.$ It turns out that, for a good triplet of vectors $(\mathbf{v+r},(q-1)\mathbf{r},\mathbf{v}+q\mathbf{r})$, we can obtain $\langle \mathbf{v},\beta^i \rangle$ for all $i \in [L]$.

 Furthermore, it follows from Lemma \ref{ref:gmm} that
 for a query vector $\vv$  with integral entries, a batch size of $T>c_3\log n\exp(({\sigma}/{\epsilon})^{2/3})$ , for some constant $c_3>0$, is sufficient to recover the denoised query responses $\langle \mathbf{v},\beta^1 \rangle, \langle \mathbf{v},\beta^2 \rangle,\dots, \langle \mathbf{v},\beta^L \rangle$ for all the queries  with probability at least $1-{1}/{\textup{poly}(n)}$.

\paragraph{The Alignment Step.}
Let a particular good query vector triplet be $(\mathbf{v^{\star}+r^{\star}},(q^{\star}-1)\mathbf{r^{\star}},\mathbf{v}^{\star}+q^{\star}\mathbf{r^{\star}})$. From now, we will consider the $L$ elements $\{\langle \mathbf{ r^{\star}},\beta^i \rangle\}_{i=1}^{L}$
to be labels and for a vector $\mathbf{u}$, we will associate a label with every element in $\{\langle \mathbf{u},\mathbf{\beta}^i \rangle\}_{i=1}^{L}$. The labelling is correct if, for all $i \in [L]$, the element labelled as $\langle \mathbf{r^{\star}},\beta^i \rangle$ also corresponds to the same unknown vector $\beta^i$. Notice that we can label the elements $\{\langle \mathbf{v^{\star}},\mathbf{\beta}^i \rangle\}_{i=1}^{L}$ correctly because the triplet $(\mathbf{v^{\star}+r^{\star}},(q^{\star}-1)r^{\star},\mathbf{v^{\star}}+q^{\star}\mathbf{r^{\star}})$ is good. Consider another good query vector triplet $(\mathbf{v'+r'},(q'-1)\mathbf{r'},\mathbf{v'}+q'\mathbf{r'})$. This {\em matches} with the earlier query triplet  if  additionally, the vector triplet $(\mathbf{r'},\mathbf{r^{\star}},\mathbf{r'}+\mathbf{r^{\star}})$ is also good. 

Such matching pair of good triplets exists, and can be found by random choice with some probability. We show that, the matching good triplets allow us to do the alignment in the case of general $L>2.$

At this point we would again like to appeal to the standard compressed sensing results. However we need to show that the matching good vectors themselves form a matrix that has the required RIP property. As our final step, we establish this fact.

\begin{rmk}[Refinement and adaptive queries]
It is possible to have a sample complexity of $O\Big(k (\log n)^2\log k\exp \Big((\epsilon\sqrt{\mathsf{SNR}})^{-2/3} \Big)\Big)$ in Theorem~\ref{thm:bigone}, but with a probability of $1-{\rm poly}(k^{-1}).$ Also it is possible to shave-off another $\log n$ factor from sample complexity if we can make the queries adaptive.
\end{rmk}

\vspace{0.1in}
{\em Acknowledgements:} This research is supported in part by NSF Grants CCF 1642658, 1618512, 1909046,  1908849 and 1934846.

\clearpage

\appendix

\begin{center}
{\Large Supplementary Material: Sample Complexity of Learning Mixtures  of Sparse Linear Regressions}    
\end{center}

\section{Proof of Theorem \ref{thm:noiselesslb}}
It is known that for any particular vector $\beta$, at least $2k$ queries to the oracle are necessary in order to recover the vector exactly. Suppose the random variable $X$ denotes the number of queries until the oracle has sampled the vector $\beta$ at least $2k$ times. Notice that $X=\sum_{i=1}^{2k} X_i$ can be written as a sum of independent and identical random variables $X_i$ distributed according to the geometric distribution with parameter ${1}/{L}$ where $X_i$ denotes the number of attempts required to obtain the $i^{\textup{th}}$ sample after the $(i-1)^{\textup{th}}$ sample has been made by the oracle. Since $X$ is a sum of independent random variables, we must have
\begin{align*}
\avg X=2Lk \quad \textup{and} \quad \textup{Var}(X)=2k(L^2-L)
\end{align*}
Therefore by using Chebychev's inequality \cite{boucheron2013concentration}, we must have
\begin{align*}
\Pr\left(X\le 2Lk-k^{\frac{1}{4}}\sqrt{2k(L^2-L)}\right) \le \frac{1}{\sqrt{k}}
\end{align*}
and therefore $X>2Lk(1-o(1))$ with high probability which proves the statement of the theorem.

\section{Description of Algorithm \ref{alg:noisy} and Proof of Theorem \ref{thm:noisy}}
\label{app:noisy_two}
 \noindent \textbf{Algorithm \ref{alg:noisy} (Design of queries and denoising):} Let $m$ be the total number of queries that we will make. In the first step of the algorithm, for a particular query vector $\mathbf{v}\in \mathbb{R}^n$, our objective is to recover $\langle \mathbf{v},\beta^1 \rangle$ and $\langle \mathbf{v},\beta^2 \rangle$ which we will denote as the \textit{denoised query responses} corresponding to the vector $\mathbf{v}$. It is intuitive, that in order to do this, we need to use the same query vector $\mathbf{v}$ repeatedly a number of times and aggregate the noisy query responses to recover the denoised counterparts.

Therefore, at every iteration in Step \texttt{1} of Algorithm \ref{alg:noisy}, we sample a vector $\mathbf{v}$ uniformly at random from $\{+1,-1\}^{n}$. Once the vector $\mathbf{v}$ is sampled, we use $\mathbf{v}$ as query vector repeatedly for $T$ times. We will say that the query responses to the same vector as query to be a \textit{batch} of size  $T$. It can be seen that since $\mathbf{v}$ is fixed, the query responses in a batch is sampled from a Gaussian mixture distribution $\mathcal{M}$ with means $\langle \mathbf{v},\beta^1 \rangle$ and $\langle \mathbf{v},\beta^2 \rangle$ and variance $\sigma^2$,  in short,
\begin{align*}
\mathcal{M}=\frac{1}{2}\mathcal{N}(\langle \mathbf{v},\beta^1 \rangle,\sigma^2)+\frac{1}{2}\mathcal{N}(\langle \mathbf{v},\beta^2 \rangle,\sigma^2).
\end{align*}
Therefore the problem reduces to recovering the mean parameters from a mixture of Gaussian distribution with at most two mixture constituents (since the means can be same) and having the same variance. We will use the following important lemma for this problem. 

\begin{lem*}[Lemma~\ref{ref:gmm}: Learning Gaussian mixtures] 
Let $\mathcal{M}=\frac{1}{L}\sum_{i=1}^L\mathcal{N}(\mu_i,\sigma^2)$
be a uniform mixture of $L$ univariate Gaussians, with known shared
variance $\sigma^2$ and with means $\mu_i \in
\epsilon\mathbb{Z}$. Then, for some constant $c>0$ and some $t=\omega(L)$, there exists an algorithm that requires $ ctL^2 \exp((\sigma/\epsilon)^{2/3})$ samples from $\mathcal{M}$ and
exactly identifies the parameters $\{\mu_i\}_{i=1}^L$ with
probability at least $1-2e^{-2t}$.
\end{lem*}
The proof of this lemma can be found in Appendix~\ref{sec:gmm}.
We now have the following lemma to characterize the size of each batch $T$.
\begin{lem}{\label{lem:gmm}}
For any query vector $\mathbf{v} \in \{+1,0,-1\}^{n}$, a batchsize of $T=c_1\log n\exp(({\sigma}/{\epsilon})^{2/3})$, for a constant $c_1>0$, is sufficient to recover the denoised query responses $\langle \mathbf{v},\beta^1 \rangle$ and $\langle \mathbf{v},\beta^2 \rangle$ with probability at least $1-{1}/{\textup{poly}(n)}$.
\end{lem}
\begin{proof}
 Since $\mathbf{v}\in \{+1,0,-1\}^{n}$,  $\langle \mathbf{v},\beta^1 \rangle,\langle \mathbf{v},\beta^2 \rangle \in \epsilon\mathbb{Z}.$  Using Lemma \ref{ref:gmm}, 
 the claim follows.
\end{proof}
\begin{coro}
For any $O\Big(k\log n\log ({n}/{k})\Big)$ query vectors sampled uniformly at random from $\{+1,-1\}^{n}$, a batch size of $T>c_2\log n\exp((\frac{\sigma}{\epsilon})^{2/3})$, for some constant $c_2>0$, is sufficient to recover the denoised query responses corresponding to every query vector with probability at least $1-{1}/{\textup{poly}(n)}$. 
\end{coro}
\begin{proof}
This statement is proved by taking a union bound over $O\Big(k\log n\log ({n}/{k})\Big)$ batches corresponding to that many query vectors.
\end{proof}
\noindent \textbf{Algorithm \ref{alg:noisy} (Alignment step):} 
Notice  from the previous discussion, for each batch corresponding to a query vector $\mathbf{v}$, we obtain the pair of values $(\langle \mathbf{v},\beta^1 \rangle,\langle \mathbf{v},\beta^2 \rangle)$. However, we still need to cluster these values (by taking one value from each pair and assigning it to one of the clusters) into two clusters corresponding to $\beta_1$ and $\beta_2$. We will first explain the clustering process for two particular query vectors $\mathbf{v}_1$ and $\mathbf{v}_2$ for which we have already obtained the pairs $(\langle \mathbf{v}_1,\beta^1 \rangle,\langle \mathbf{v}_1,\beta^2 \rangle)$ and $(\langle \mathbf{v}_2,\beta^1 \rangle,\langle \mathbf{v}_2,\beta^2 \rangle)$.
The objective is to cluster the four samples into two groups of two samples each so that the samples in each cluster correspond to the same unknown sensed vector. 
 Now, we have two cases to consider: \\
\noindent \textbf{Case 1:} $(\langle \mathbf{v}_1,\beta^1 \rangle=\langle \mathbf{v}_1, \beta^2 \rangle \textup{ or } \langle \mathbf{v}_2,\beta^1 \rangle=\langle \mathbf{v}_2,\beta^2 \rangle)$ In this scenario, the values in at least one of the pairs are same and any grouping works. 

\noindent \textbf{Case 2:} $(\langle \mathbf{v}_1,\beta^1 \rangle \neq \langle \mathbf{v}_1 \beta^2 \rangle \textup{ and } \langle \mathbf{v}_2,\beta^1 \rangle \neq \langle \mathbf{v}_2,\beta^2 \rangle)$. We use two more batches corresponding to the vectors $\frac{\mathbf{v}_1+\mathbf{v}_2}{2}$ and $\frac{\mathbf{v}_1-\mathbf{v}_2}{2}$ which belong to $\{-1,0,+1\}^n$. We will call the vector $\frac{\mathbf{v}_1+\mathbf{v}_2}{2}$ the \textit{sum query} and the vector $\frac{\mathbf{v}_1-\mathbf{v}_2}{2}$ the \textit{difference query} corresponding to $\mathbf{v}_1,\mathbf{v}_2$ respectively. Hence using Lemma \ref{lem:gmm} again, we will be able to obtain the pairs $(\langle \frac{\mathbf{v}_1+\mathbf{v}_2}{2},\beta^1 \rangle,\langle \frac{\mathbf{v}_1+\mathbf{v}_2}{2},\beta^2 \rangle)$ and $(\langle \frac{\mathbf{v}_1-\mathbf{v}_2}{2},\beta^1 \rangle,\langle \frac{\mathbf{v}_1-\mathbf{v}_2}{2},\beta^2 \rangle)$. 
Now, we will choose two elements from the pairs $(\langle \mathbf{v}_1,\beta^1 \rangle,\langle \mathbf{v}_1 \beta^2 \rangle)$ and $(\langle \mathbf{v}_2,\beta^1 \rangle,\langle \mathbf{v}_2 \beta^2 \rangle)$ (one element from each pair) such that their sum belongs to the pair $2\langle \frac{\mathbf{v}_1+\mathbf{v}_2}{2},\beta^1 \rangle,2\langle \frac{\mathbf{v}_1+\mathbf{v}_2}{2},\beta^2 \rangle$ and their difference belongs to the pair $2\langle \frac{\mathbf{v}_1-\mathbf{v}_2}{2},\beta^1 \rangle,2\langle \frac{\mathbf{v}_1-\mathbf{v}_2}{2},\beta^2 \rangle$. In our algorithm, we will put these two elements into one cluster and the other two elements into the other cluster.
From construction, we must put $(\langle \mathbf{v}_1,\beta^1 \rangle, \langle \mathbf{v}_2,\beta^1 \rangle)$ in one cluster and $(\langle \mathbf{v}_1,\beta^2 \rangle, \langle \mathbf{v}_2,\beta^2 \rangle)$
in other.

Putting it all together, in Algorithm \ref{alg:noisy}, we uniformly and randomly choose $c_sk\log \frac{n}{k}$ query vectors from $\{+1,-1\}^{n}$ and for each of them, we use it repeatedly for $c_2\log n\exp \Big(\frac{\sigma}{\epsilon}\Big)^{2/3}$ times. From each batch, we recover the denoised query responses for the query vector associated with that batch. For a particular query vector $\mathbf{v}$, we call the query vector \textit{good} if $\langle \mathbf{v},\beta^1 \rangle \neq \langle \mathbf{v},\beta^2 \rangle$.
For a  $\mathbf{v}$ chosen uniformly at randomly from $\{+1,-1\}^{n}$, the probability that $\langle \mathbf{v},\beta^1-\beta^2 \rangle=0$ is at most $\frac{1}{2}$. Therefore, if one chooses  $\log n$ query vectors  uniformly and independently at random from $\{+1,-1\}^n$, at least one is good with probability $1-\frac{1}{n}$.
\remove{
We now have the following lemma.
\begin{lem}\label{lem:good}
Out of $\log n$ query vectors sampled uniformly and independently at random from $\{+1,-1\}^n$, at least one is good with probability $1-\frac{1}{n}$.
\end{lem}
\begin{proof}
\akshay{This can also be significantly condensed or moved to the appendix. After saying that $\Pr(\textrm{random query vector is not good}) \leq 1/2$, you can just say the result then follows from Chernoff bound.}
Consider a vector $\mathbf{v}$ chosen uniformly at randomly from $\{+1,-1\}^{n}$. Since all entries of $\mathbf{v}$ are chosen independently and there are two equiprobable choices for every entry of the vector, the probability that $\langle \mathbf{v},\beta^1-\beta^2 \rangle=0$ is at most $\frac{1}{2}$.
\begin{align*}
\Pr(\textup{A random query vector is not good})  \le \frac{1}{2}
\end{align*}
The above probability is tight in the case when $\beta^1-\beta^2$ has only two non-zero entries that are equal in magnitude but has opposite signs. Hence, 
\begin{align*}
\Pr(\textup{All the $\log n$ chosen query vectors are not good})  \le \frac{1}{2^{\log n}} \le \frac{1}{n}
\end{align*}
which proves the statement of the lemma.
\end{proof} }
We are now ready to prove the main theorem. 
\begin{proof}[Proof of Theorem~\ref{thm:noisy}]
For each vector $\mathbf{v}$ belonging to the set of first $\log n$ query vectors and for each query vector $\mathbf{b}$ ($\mathbf{b}$ is among the initial $c_sk\log \frac{n}{k}$ query vectors) different from $\mathbf{v}$, we make two additional batches of queries corresponding to query vectors $\frac{\mathbf{v}+\mathbf{b}}{2}$ and $\frac{\mathbf{v}-\mathbf{b}}{2}$. Consider the first $\log n$ query vectors. We know that one of them, say $\mathbf{g}$, is a good query vector. Let us denote the denoised means obtained from the batch of queries corresponding to $\mathbf{g}$ to be $(x,y)$. We can think of $x$ and $y$ as labels for the clustering of the denoised means from the other query vectors. Now, from the alignment step, we know that for every query vector $\mathbf{b}$ different from $\mathbf{g}$ and the denoised query responses $(p,q)$ corresponding to $\mathbf{b}$, by using the additional sum and difference queries, we can label one of the element in $(p,q)$ as $x$ and the other one as $y$. Since the vector $\mathbf{g}$ is good, therefore $x\neq y$ and hence we will be able to aggregate the denoised query responses corresponding to $\beta^1$ and the denoised query responses corresponding to $\beta^2$ separately. Since we have $c_sk \log n$ query responses for each of $\beta^1$ and $\beta^2$, we can scale the query responses by a factor of $1/\sqrt{c_s k\log n}$ and subsequently, we can run basis pursuit~\cite{candes2008restricted}  to recover the best $k$-sparse approximations of both $\beta^1$ and $\beta^2$. Notice that the total number of queries in this scheme is $O(k\log^2 n)$ and since the size of each batch corresponding to each query is $O(\log n\exp((\frac{\sigma}{\epsilon})^{2/3}))$, the total sample complexity required is $O\Big(k (\log n)^{3} \exp \Big( \frac{\sigma}{\epsilon} \Big)^{2/3}\Big)$. 
\end{proof}

\section{Proof of Lemma \ref{ref:gmm}}\label{sec:gmm}
Note that Lemma \ref{ref:gmm} is \emph{not} a claimed contribution of this paper. Rather, it appears as one of the results in another submission (to a different conference). Since we can't cite this other paper yet we include the details here for completeness. 

\begin{lem}\label{lem:chartv}
For any two distributions $f,f'$ defined over the same sample space $\Omega\subseteq\mathbb{R}$, we have
\begin{align*}
\variation{f -f'} \ge \frac{1}{2} \sup_{t \in \reals}|C_f(t) -C_f'(t)|.
\end{align*}
More generally, for any $G: \Omega \to \mathbb{C}$ and $\Omega' \subset \Omega$ we have
\begin{align*}
\variation{f -f'} \geq \left(2\sup_{x \in \Omega'}|G(x)|\right)^{-1} \Big(\left| \avg_{X \sim f} G(X) - \avg_{X \sim f'}G(X') \right| \\\qquad \,\,- \int_{x \in\Omega \setminus \Omega'}|G(x)| \cdot |df(x) - df'(x)|\Big).
\end{align*}
\end{lem}
\begin{proof}
We prove the latter statement, which implies the former since for the
function $G(x) = e^{itx}$ we have $\sup_{x} |G(x)| = 1$. By the
triangle inequality we have
\begin{align*}
|\avg_{X \sim f} G(X) &- \avg_{X \sim f'}G(X)| \leq \int_{x \in \Omega}|G(x)| \cdot |df(x) - df'(x)| \\
& \le 2\sup_{x \in \Omega'}|G(x)| \cdot \variation{f -f'} + \int_{x \in\Omega \setminus \Omega'}|G(x)| \cdot |df(x) - df'(x)|. \tag*\qedhere
\end{align*}
\end{proof}
\begin{lem}\label{lem:gaussian}
Let $z=\exp(i t)$ where $t \in [-\pi/L,\pi /L]$. If the random variable $X\sim \mathcal{N}(\mu,\sigma)$ and $G_t(x) = e^{itx}$ then 
\[\avg [G_t(X)]=\exp(-\sigma^2 t^2/2)z^\mu \mbox{ and } \|G_t\|_{\infty}= 1 \ . \]
\end{lem}
\begin{proof}
Observe that $\avg[G_t(X)]$ is precisely the characteristic function. Clearly we have $\|G_t\|_{\infty} = 1$ and further
\[\avg[G_t(X)]=\exp(it \mu -\sigma^2 t^2/2)=\exp(-\sigma^2 t^2/2) z^\mu.\]
\end{proof}
We crucially use the following lemma. 
\begin{lem}[\cite{BorweinE97}] \label{lem:npb} Let $a_0, a_1, a_2, \dots \in \{-L,-(L-1),\dots,L-1,L\}$ be such that not all of them are zero. For any complex number $z$, let $A(z) \equiv \sum_{\ell}  a_{\ell} z^{\ell}.$ Then, for some absolute constant $c$, 
$$
\max_{-\pi/S \le t \le \pi/S} |A(e^{it})| \ge e^{-cS} \ . 
$$
\end{lem}

\begin{lem}[TV Lower Bounds] \label{lem:tv} Consider two mixtures of Gaussian distributions such that $\mathcal{M} =
  \frac{1}{L}\sum_{i=1}^L \mathcal{N}(\mu_i,\sigma)$
and 
$\mathcal{M}' =
  \frac{1}{L}\sum_{i=1}^L \mathcal{N}(\mu_i',\sigma)$
where 
$\mu_i, \mu_i' \in \epsilon \mathbb{Z}$.
Then
\[
\variation{\mathcal{M}' -\mathcal{M}} \geq L^{-1} \exp(-\Omega((\sigma/\epsilon)^{2/3})).
\]
\end{lem}
\begin{proof}
The characteristic function of a Gaussian $X \sim \mathcal{N}(\mu,\sigma^2)$ is 
\begin{align*}
C_\mathcal{N}(t)=\avg e^{it X}=e^{it\mu-\frac{t^{2}\sigma^{2}}{2}}.
\end{align*}   
Therefore we have that 
\begin{align*}
C_\mathcal{M}(t)-C_{\mathcal{M}'}(t) \ge \frac{e^{-\frac{t^{2}\sigma^{2}}{2}}}{L}    \sum_{i =1}^L (e^{it \mu_i}-  e^{it \mu_i'}) .
\end{align*}
Now, using Lemma~\ref{lem:npb}, there exist an absolute constant $c$ such that,  
\begin{align*}
\max_{-\frac{\pi}{\epsilon S}\le t \le \frac{\pi}{\epsilon S}} \big| \sum_{i =1}^L (e^{it \mu_i}-  e^{it \mu_i'})\big| \ge e^{-cS}.
\end{align*}
Also, for $t \in (-\frac{\pi}{\epsilon S}, \frac{\pi}{\epsilon S}),$ $e^{-\frac{t^{2} \sigma^{2}}{2}} \ge e^{-\frac{\sigma^2\pi^2}{2\epsilon^2 S^2}}.$ And therefore,
\begin{align*}
\Big|C_\mathcal{M}(t)-C_{\mathcal{M}'}(t)\Big| \ge \frac{1}{L}  e^{-\frac{\sigma^2\pi^2}{2\epsilon^2 S^2}-cS}.
\end{align*}
By substituting $S = \frac{(\pi\sigma)^{2/3}}{(\epsilon^2c)^{1/3}}$ above we conclude that there exists $t$ such that 
\begin{align*}
\Big|C_\mathcal{M}(t)-C_{\mathcal{M}'}(t)\Big| \ge \frac{1}{L}  e^{-\frac32(c\pi\sigma/\epsilon)^{2/3}}.
\end{align*}
Now using 
 Lemma~\ref{lem:chartv}, we have
$\variation{\mathcal{M}' -\mathcal{M}} \geq L^{-1} \exp(-\Omega((\sigma/\epsilon)^{2/3}))$.
\end{proof}
To learn the parameters of a
 Gaussian mixture
 \[\mathcal{M} =
  \frac{1}{L}\sum_{i=1}^L \mathcal{N}(\mu_i,\sigma) ~~\mbox{ where }~~\mu_i \in \{\ldots, -2\epsilon, -\epsilon, 0,\epsilon, 2\epsilon \ldots \} \]
  we use the minimum distance estimator precisely defined in \cite[Section~6.8]{devroye2012combinatorial}. Let $\mathcal{A} \equiv \{\{x: \mathcal{M}(x) \ge \mathcal{M}'(x)\}: \text{ for any two mixtures } \mathcal{M} \ne \mathcal{M}'\}$ be a collection of subsets. Let $P_m$ denote the empirical probability measure induced by the $m$ samples. Then, choose a mixture $\hat{\mathcal{M}}$ for which 
  the quantity $\sup_{A\in \mathcal{A}}  |\Pr_{\sim\hat{\mathcal{M}}}(A) - P_m(A) |$ is minimum (or within $1/m$ of the infimum). This is the minimum distance estimator, whose performance is guaranteed by the following proposition~\cite[Thm.~6.4]{devroye2012combinatorial}.
  

  
  \begin{lem}
    Given $m$ samples from $\mathcal{M}$ and with $\Delta = \sup_{A \in \mathcal{A}}|\Pr_{\sim \mathcal{M}}(A) - P_m(A)|$, we have
  $$
  \variation{\hat{\mathcal{M}} -\mathcal{M}} \le 4\Delta +\frac{3}{m}.
  $$
\end{lem}
We now upper bound the right-hand side of the above inequality. It is known that the mean of $\Delta$ is bounded from above by a function of  $VC(\mathcal{A})$, the VC dimension of the class $\mathcal{A}$, see \cite[Section~4.3]{devroye2012combinatorial} and is given by
\begin{align*}
\avg \Delta \le c_2 \sqrt{\frac{VC(\mathcal{A})}{m}} \quad \textup{for some universal constant } c_2>0
\end{align*}
Now, via McDiarmid's inequality and a standard symmetrization argument, $\Delta$ is concentrated around its mean, see
\cite[Section~2.4]{devroye2012combinatorial}:  and therefore, for some $t>0$
\begin{align*}
\Delta \le \avg \Delta+\sqrt{\frac{t}{m}} 
\end{align*}
with probability at least $1-2e^{-2t}$.
Therefore, we must have
$$
\variation{\hat{\mathcal{M}} - \mathcal{M}} \leq 4\Delta + O(1/m) \leq 4\avg_{\sim\mathcal{M}} \Delta +\sqrt{\frac{t}{m}}+ o(1/\sqrt{m}) \le 4\sqrt{\frac{VC(\mathcal{A})}{m}}+\sqrt{\frac{t}{m}},
$$ with probability at least $1-2e^{-2t}$.  This first term is bounded by the following: 

 
\begin{lem}
For the class $\mathcal{A}$ defined above, the VC dimension is given by $VC(\mathcal{A}) = O(L)$.
\end{lem}
\begin{proof}
First of all we show that any element of the set $\mathcal{A}$ can be written
as union of at most $4L-1$ intervals in $\mathbb{R}$. For this we use the
fact that a linear combination of $L$ Gaussian pdfs $f(x) =
\sum_{i=1}^{L} \alpha_{i} f_{i}(x)$ where $f_i$s normal pdf
$\mathcal{N}(\mu_i,\sigma^2_i)$ and $\alpha_i \in \reals, 1\le i\le L$ has at
most $2L-2$ zero-crossings \cite{kalai2012disentangling}. Therefore,
for any two mixtures of interest $\mathcal{M}(x) -\mathcal{M}'(x)$ has at most $4L-2$
zero-crossings. Therefore any $A\in \mathcal{A}$ must be a union of at most
$4L-1$ contiguous regions in $\mathbb{R}$.  It is now an easy exercise to
see that the VC dimension of such a class is $\Theta(L)$.
\end{proof}
%
%


As a result, when $t=\omega(L)$  the error of the minimum distance estimator is less $2\sqrt{\frac{t}{m}}$ with probability at least $1-2e^{-2t}$. 
But from lemma~\ref{lem:tv}, notice that for any other mixture $\mathcal{M}'$ we must have,
$$
\variation{\mathcal{M} -\mathcal{M}'} \ge  L^{-1} \exp(-\Omega((\sigma/\epsilon)^{2/3})).
$$ 
As long as $\variation{\hat{\mathcal{M}} -\mathcal{M}} \le \frac12 \variation{\mathcal{M} -\mathcal{M}'}$ we will exactly identify the parameters. Therefore, for some universal constant $c'>0$, 
$m = c'tL^2 \exp((\sigma/\epsilon)^{2/3})$ samples suffice to exactly learn the parameters with probability at least $1-2e^{-2t}$.
\section{Analysis of Algorithm~\ref{alg:noisygen} for General $L$ and Proof of Theorem ~\ref{thm:bigone} and Theorem ~\ref{thm:noisygen}}\label{sec:ngenL}
\noindent \textbf{Algorithm \ref{alg:noisygen} (Design of queries):} In every iteration in Step \texttt{1} of Algorithm \ref{alg:noisygen}, we will sample a vector $\mathbf{v}$ uniformly at random from $\{+1,-1\}^n$, another vector $\mathbf{r}$ uniformly at random from $\calG \equiv \{-2z^{\star},-2z^{\star}+1,\dots,2z^{\star}-1,2z^{\star}\}^n$ and a number $q$ uniformly at random from $\{1,2,\dots,4z^{\star}+1\}$. Now, we will use a batch of queries corresponding to the vectors $\mathbf{v+r},(q-1)\mathbf{r}$ and $\mathbf{v}+q\mathbf{r}$. We have the following lemmas describing several necessary properties of such queries.\\
\noindent We will define a triplet of query vectors $(\mathbf{v}_1,\mathbf{v}_2,\mathbf{v}_3)$ to be \textit{good} if for all triplets of indices $i,j,k \in [L]$ such that $i,j,k$ are not identical, it must happen that 
\begin{align*}
    \langle \mathbf{v}_1,\beta^i \rangle+\langle \mathbf{v}_2,\beta^j \rangle \neq \langle \mathbf{v}_3,\beta^k \rangle
\end{align*}
\begin{lem}\label{lem:goodtrip}
The query vector triplet $(\mathbf{v+r},(q-1)\mathbf{r},\mathbf{v}+q\mathbf{r})$ is good with probability at least $\frac{1}{\sqrt{\alpha^{\star}}}$. 
\end{lem}
\begin{proof}
Notice that for a fixed triplet $i,j,k \in [L]$ such that $i,j,k$ are not identical, we must have
\begin{align*}
    &\Pr(\langle \mathbf{v+r},\beta^i \rangle+\langle (q-1)\mathbf{r},\beta^j \rangle=\langle \mathbf{v}+q\mathbf{r}, \beta^k \rangle) \\
    &=\Pr(\langle \mathbf{r},\beta^i+(q-1)\beta^j-q\beta^k \rangle=\langle \mathbf{v},\beta^k-\beta^i \rangle)\\
    &\le \Pr(\beta^i+(q-1)\beta^j-q\beta^k = 0)+\Pr(\beta^i+(q-1)\beta^j-q\beta^k \neq 0 )\\
    &\cdot\Pr(\langle \mathbf{r},\beta^i+(q-1)\beta^j-q\beta^k \rangle=\langle \mathbf{v},\beta^k-\beta^i \rangle\mid \beta^i+(q-1)\beta^j-q\beta^k \neq 0)  \\
    &\le \Big(1-\frac{1}{4z^{\star}+1}\Big)\frac{1}{4z^{\star}+1}+\frac{1}{4z^{\star}+1}=\frac{2}{4z^{\star}+1}-\frac{1}{(4z^{\star}+1)^{2}}.
\end{align*}
{\color{black}Notice that $\beta^i+(q-1)\beta^j-q\beta^k=0$ cannot hold for two values of $q: q_1$ and $q_2$. We will show this fact by contradiction. Suppose it happens that $\beta^i+(q_1-1)\beta^j-q_1\beta^k=0$ and $\beta^i+(q_2-1)\beta^j-q_2\beta^k=0$ in which case we must have $\beta^j=\beta^k$ which is a contradiction to the fact that all the unknown vectors are distinct.}
We can take a union over all possible triplets (at most $L^3$ of them) and therefore we must have that 
\begin{align*}
    \Pr(\textup{The vector triplet $(\mathbf{v+r},(q-1)\mathbf{r},\mathbf{v}+q\mathbf{r})$ is good})& \ge 1-L^3\Big(\frac{2}{4z^{\star}+1}-\frac{1}{(4z^{\star}+1)^{2}}\Big) \\
    &\ge \frac{1}{\sqrt{\alpha^{\star}}}.
\end{align*}
\end{proof}
 We will now generalize Lemma \ref{lem:gmm} in order to characterize the batch size required to recover the denoised query responses when there are $L$ unknown vectors that the oracle can sample from. 
\begin{lem}[Generalization of Lemma \ref{lem:gmm}]
For a particular query vector $\mathbf{v}$ such that each entry of $\mathbf{v}$ is integral, a batch size of $T>c_3\log n\exp(({\sigma}/{\epsilon})^{2/3})$, for some constant $c_3>0$, is sufficient to recover the denoised query responses $\langle \mathbf{v},\beta^1 \rangle, \langle \mathbf{v},\beta^2 \rangle,\dots, \langle \mathbf{v},\beta^L \rangle$  with probability at least $1-{1}/{\textup{poly}(n)}$.
\end{lem}
\begin{proof}
The proof follows in exactly the same manner as the proof in Lemma \ref{lem:gmm} but in this case, we invoke Lemma \ref{ref:gmm} with any general value of $L$. Since we have assumed that $L$ is a constant, the term $L^2$ is subsumed within the constant $c_3$. 
\end{proof}
\begin{coro}{\label{coro:imp}}
For $O\Big(k\log^2 n\Big)$ query vectors such that every entry of every query vector is integral, a batch size of $T>c_4\log n\exp(({\sigma}/{\epsilon})^{2/3})$, for some constant $c_4>0$, is sufficient to recover the denoised query responses corresponding to every query vector with probability at least $1-{1}/{\textup{poly}(n)}$. 
\end{coro}
\begin{proof}
    Again, we can take a union bound over all $O\Big(k \log^2 n\Big)$ query vectors to obtain the result.
\end{proof}
\begin{lem}\label{lem:good}
If we draw $\sqrt{\alpha^{\star}}\log n$ triplets of query vectors $(\mathbf{v+r},(q-1)\mathbf{r},\mathbf{v}+q\mathbf{r})$ randomly as described, then at least one of the triplets  is good with probability at least $1-{1}/{n}$. 
\end{lem}
\begin{proof}
 Now, the probability of a triplet of vectors $(\mathbf{v+r},(q-1)\mathbf{r},\mathbf{r})$ being not good is less than $1-\frac{1}{\sqrt{\alpha^{\star}}}$ and therefore the probability of all the $\sqrt{\alpha^{\star}}\log n$ triplets being not good is less than
\begin{align*}
 \Big(1-\frac{1}{\sqrt{\alpha^{\star}}}\Big)^{\log n\sqrt{\alpha^{\star}}} \le e^{-\log n} \le n^{-1}   
\end{align*}
which proves the statement of the lemma.
\end{proof}

\begin{lem}\label{lem:align}
For a good triplet of vectors $(\mathbf{v+r},(q-1)\mathbf{r},\mathbf{v}+q\mathbf{r})$, we can obtain $\langle \mathbf{v},\beta^i \rangle$ for all $i \in [L]$.
\end{lem}
\begin{proof}
Recall that since we queried the vector $\mathbf{v}+q\mathbf{r}$, we can simply check which element (say $x$) from the set $\{\langle \mathbf{v+r}, \beta^i \rangle\}_{i=1}^{L}$ and which element (say $y$) from the set $\{\langle (q-1)\mathbf{r}, \beta^i \rangle\}_{i=1}^{L}$ adds up to an element in $\{\langle \mathbf{v}+q\mathbf{r}, \beta^i \rangle\}_{i=1}^{L}$. It must happen that the elements $x$ and $y$ must correspond to the same unknown vector $\beta^i$ for some $i \in [L]$ because the triplet of vectors $(\mathbf{v+r},(q-1)\mathbf{r},q\mathbf{r})$ is good. Hence computing $x-(y/(q-1))$ allows us to obtain $\langle \mathbf{v},\beta^i \rangle$ and this step can be done for all $i \in [L]$. 
\end{proof}
\textbf{Algorithm \ref{alg:noisygen} (Alignment step):}
Let a particular good query vector triplet be $(\mathbf{v^{\star}+r^{\star}},(q^{\star}-1)\mathbf{r^{\star}},\mathbf{v}^{\star}+q^{\star}\mathbf{r^{\star}})$. From now, we will consider the $L$ elements $\{\langle \mathbf{ r^{\star}},\beta^i \rangle\}_{i=1}^{L}$ (necessarily distinct) 
to be labels and for a vector $\mathbf{u}$, we will associate a label with every element in $\{\langle \mathbf{u},\mathbf{\beta}^i \rangle\}_{i=1}^{L}$. The labelling is correct if, for all $i \in [L]$, the element labelled as $\langle \mathbf{r^{\star}},\beta^i \rangle$ also corresponds to the same unknown vector $\beta^i$. Notice that we can label the elements $\{\langle \mathbf{v^{\star}},\mathbf{\beta}^i \rangle\}_{i=1}^{L}$ correctly because the triplet $(\mathbf{v^{\star}+r^{\star}},(q^{\star}-1)r^{\star},\mathbf{v^{\star}}+q^{\star}\mathbf{r^{\star}})$ is good and by applying the reasoning in Lemma \ref{lem:align}. Consider another good query vector triplet $(\mathbf{v'+r'},(q'-1)\mathbf{r'},\mathbf{v'}+q'\mathbf{r'})$ which we will call \textit{matching good} with respect to  $(\mathbf{v^{\star}+r^{\star}},(q^{\star}-1)r^{\star},\mathbf{v^{\star}}+q^{\star}\mathbf{r^{\star}})$ if it is good and additionally, the vector triplet $(\mathbf{r'},\mathbf{r^{\star}},\mathbf{r'}+\mathbf{r^{\star}})$ is also good. 
\begin{lem}\label{lem:super}
For a fixed known good query vector triplet
  $(\mathbf{v^{\star}+r^{\star}},(q^{\star}-1)r^{\star},\mathbf{v^{\star}}+q^{\star}\mathbf{r^{\star}})$, the probability that any randomly drawn query vector triplet $(\mathbf{v'+r'},(q-1)\mathbf{r'},\mathbf{v'}+q'\mathbf{r'})$ is matching good with respect to  $(\mathbf{v^{\star}+r^{\star}},(q^{\star}-1)r^{\star},\mathbf{v^{\star}}+q^{\star}\mathbf{r^{\star}})$ is at least $\frac{1}{\sqrt{\alpha^{\star}}}$.
\end{lem}
\begin{proof}
From Lemma \ref{lem:goodtrip}, we know that the probability that a randomly drawn query vector triplet $(\mathbf{v'+r'},(q-1)\mathbf{r'},\mathbf{v'}+q'\mathbf{r'})$ is not good is at most $L^3\Big(\frac{2}{4z^{\star}+1}-\frac{1}{(4z^{\star}+1)^{2}}\Big)$.
Again, we must have for a fixed triplet of indices $i,j,k \in [L]$ such that they are not identical
\begin{align*}
    &\Pr(\langle \mathbf{r'},\beta^i \rangle+\langle \mathbf{r}^{\star},\beta^j \rangle=\langle \mathbf{r}'+\mathbf{r}^{\star}, \beta^k \rangle) \\
    &=\Pr(\langle \mathbf{r'},\beta^i-\beta^k \rangle =\langle \mathbf{r}^{\star}, \beta^k-\beta^j \rangle) \le \frac{1}{4z^{\star}+1}
\end{align*}    
Taking a union bound over all non-identical triplets (at most $L^3$ of them), we get that

\begin{align*}
    \Pr(\textup{$(\mathbf{r'},\mathbf{r^{\star}},\mathbf{r'}+\mathbf{r^{\star}}$  is not good}) \le \frac{L^3}{4z^{\star}+1}
\end{align*}
Taking a union bound over both the failure events, we get that 
\begin{align*}
 &   \Pr(\textup{$(\mathbf{v'+r'},(q-1)\mathbf{r'},\mathbf{v'}+q'\mathbf{r'})$ is not matching good}) \\
&     \le L^3\Big(\frac{3}{4z^{\star}+1}-\frac{1}{(4z^{\star}+1)^{2}}\Big) \\
&     \le 1- \frac{1}{\sqrt{\alpha^{\star}}} 
\end{align*}
which proves the lemma.
\end{proof}
\begin{lem}\label{lem:lab}
For a matching good query vector triplet $(\mathbf{v'+r'},(q-1)\mathbf{r'},\mathbf{v}'+q\mathbf{r}')$, we can label the elements in $\{\langle \mathbf{v}',\mathbf{\beta}^i \rangle\}_{i=1}^{L}$ correctly by querying the vector $\mathbf{r}'+\mathbf{r}^{\star}$.
\end{lem}
\begin{proof}
Since $(\mathbf{v'+r'},(q-1)\mathbf{r'},\mathbf{v}'+q\mathbf{r}')$ is good and we have also queried $\mathbf{v'}+q\mathbf{r'}$, we can partition the set of elements $\{\langle \mathbf{v'+r'},\mathbf{\beta}^i \rangle\}_{i=1}^{L}\cup \{\langle (q-1) \mathbf{r}',\mathbf{\beta}^i \rangle\}_{i=1}^{L}$ into groups of two elements each such that the elements in each group correspond to the same unknown vector $\beta^i$ as in the reasoning presented in proof of Lemma \ref{lem:align}. Again, since $(\mathbf{r'},\mathbf{r^{\star}},\mathbf{r'}+\mathbf{r^{\star}})$ is good and we have queried $\mathbf{r'+r^{\star}}$, we can create a similar partition of the set of elements $\{\langle \mathbf{r'},\mathbf{\beta}^i \rangle\}_{i=1}^{L}\cup \{\langle \mathbf{r}^{\star},\mathbf{\beta}^i \rangle\}_{i=1}^{L}$ and multiply every element by a factor of $q-1$. For each of the two partitions described above we can align two groups together (one from each partition) if both groups contain $\langle (q-1)\mathbf{r}',\mathbf{\beta}^i \rangle$ for the same $i \in L$ (the values $\langle \mathbf{r}',\mathbf{\beta}^i \rangle$ are necessarily distinct and therefore this is possible). Hence, for every $i \in [L]$, we can compute $\langle \mathbf{v}',\mathbf{\beta}^i \rangle$ correctly and also label it correctly because of the alignment.
\end{proof}

\noindent \textbf{Algorithm \ref{alg:noisygen} (Putting it all together)} First, we condition on the event that for all batches of queries (number of batches will be polynomial in $k$ and $\log n$) we make, the denoised means are extracted correctly which happens with probability at least $1-\frac{1}{n}$ by Corollary \ref{coro:imp}. As described in Algorithm \ref{alg:noisygen}, in the first step we sample a pair of vectors $(\mathbf{v},\mathbf{r})$ such that $\mathbf{v}$ is uniformly drawn from $\{-1,+1\}^{n}$ and  $\mathbf{r}$ is uniformly drawn from $\{-2z^{\star},-2z^{\star}+1,\dots,2z^{\star}-1,2z^{\star}\}^n$. We also sample a random number $q$ uniformly and independently from the set $\{1,2,\dots,4z^{\star}+1\}$ and  subsequently, we use batches of queries of size $c_4L^2\log n\exp((\sigma/\epsilon)^{2/3})$ corresponding to the three vectors $\mathbf{v+r},(q-1)\mathbf{r}$ and $\mathbf{v}+q\mathbf{r}$ respectively. We will repeat this step for $\sqrt{\alpha^{\star}}\log n+c'\alpha^{\star}k\log( n/k)$ iterations. Additionally, for each query vector pair $(\mathbf{(v_1,r_1)}$ among the first $\sqrt{\alpha^{\star}}\log n$ iterations and for each vector pair $(\mathbf{(v_2,r_2)}$ among the latter $c'\alpha^{\star}k\log( n/k)$ iterations, we also make the batch of queries corresponding to the vector $\mathbf{r_1+r_2}$. From Lemma \ref{lem:good}, we know that with probability at least $1-\frac{1}{n}$, one of the query vector triplets among the first $\sqrt{\alpha^{\star}}\log n$ triplets is good. Moreover, it is also easy to check if a query vector triplet is good or not and therefore it is easy to identify one. Once a good query vector triplet $(\mathbf{v^{\star}+r^{\star}},(q^{\star}-1)\mathbf{r^{\star}},\mathbf{v}^{\star}+q^{\star}\mathbf{r^{\star}})$ is identified, it is also possible to correctly identify matching good query vectors among the latter $c'\alpha^{\star}k\log (n/k)$ query vector triplets with respect to the good vector triplet. We now have the following lemma characterizing the number of matching good query vector triplets:

\begin{lem}\label{lem:supergood}
The number of matching good query vector triplets from $\alpha^{\star}c'k\log (n/k)$ randomly chosen triplets is at least $c'k\log(n/k)$ with probability at least $1-\Big(\frac{k}{n}\Big)^{\tilde{c}k}$ for some constant $\tilde{c}>0$.
\end{lem}
\begin{proof}
For a randomly drawn query vector triplet, we know that it is matching good with probability at least $\frac{1}{\sqrt{\alpha^{\star}}}$ from Lemma \ref{lem:super}.  Since there are $\alpha^{\star}c'k\log( n/k)$ query vector triplets drawn at random independently, the expected number of matching-good triplets is at least $\sqrt{\alpha^{\star}}c'k\log (n/k)$. Further, by using Chernoff bound \cite{boucheron2013concentration}, we can show that
\begin{align*}
&\Pr(\textup{Number of matching good triplets}<c'k\log (n/k)) \\
&=\Pr(\textup{Number of matching good triplets}<\sqrt{\alpha^{\star}}c'k\log (n/k)\left(1-\frac{\sqrt{\alpha^{\star}}-1}{\sqrt{\alpha^{\star}}}\right) \\
&\le \exp\left(-\frac{(\sqrt{\alpha^{\star}}-1)^{2}c'k\log (n/k)}{2\sqrt{\alpha^{\star}}}\right).
\end{align*}
\end{proof}

From Lemma \ref{lem:lab}, we know that for every matching good query vector triplet $(\mathbf{v'+r'},(q-1)\mathbf{r'},\mathbf{v}'+q\mathbf{r}')$, we can label the elements in $\{\langle \mathbf{v}',\mathbf{\beta}^i \rangle\}_{i=1}^{L}$ correctly and from Lemma \ref{lem:supergood}, we know that we have aggregated  de-noised query measurements corresponding to $c'k\log (n/k)$ vectors randomly sampled from $\{+1,-1\}^{n}$. However, since we have specifically picked $c'k\log (n/k)$ matching good vectors after the entire scheme, we do not know which query vectors will be matching good apriori and therefore we need to have the following guarantee:
\begin{lem}{\label{lem:RIPunion}}
From $\alpha^{\star}c' k\log (n/k)$ vectors randomly chosen from $\{+1,-1\}^{n}$, any $c'k\log (n/k)$ vectors scaled by a factor of $1/\sqrt{c'k\log (n/k)}$ will satisfy the $\delta-\mathsf{RIP}$ property with high probability.
\end{lem}
The proof of this lemma is delegated to the appendix. Now we are ready to proof the main theorem in this setting.

\begin{proof}[Proof of Theorem~\ref{thm:noisygen}]
The total number of batches of queries made is at most $3c'\alpha^{\star} k\log (n/k) \log n$. Further, recall that size of each batch that is sufficient to recover the denoised means accurately is $c_4\log n\log (\sigma/\epsilon)^{2/3}$. Hence the total number of queries is $O\Big(k (\log n)^{3}\exp ( \sigma/\epsilon)^{2/3}\Big)$ as mentioned in the theorem statement. From Lemma \ref{lem:lab} and Lemma \ref{lem:RIPunion}, we know that for every vector $\{\beta^i\}_{i=1}^L$, we have $c'k \log (n/k)$ linear query measurements such that the measurement matrix scaled by $1/\sqrt{c'k \log (n/k)}$ has the $\delta-\mathsf{RIP}$ property. Therefore, it is possible to obtain the best $k$-sparse approximation of all the vectors $\beta^1,\beta^2,\dots,\beta^L$ by using efficient algorithms such as Basis Pursuit. 
\end{proof}
Now  Theorem~\ref{thm:bigone} follows as a corollary.
\begin{proof}[Proof of Theorem~\ref{thm:bigone} for general $L$]
Notice that the query with the largest magnitude of query response that we will make is $\mathbf{v}+(4z^{\star}+1)\mathbf{r}$ where $\mathbf{v}$ is sampled from $\{+1,-1\}^n$ and $\mathbf{r}$ is sampled from $\{-2z^{\star},-2z^{\star}+1,\dots,2z^{\star}-1,2z^{\star}\}$. Therefore, we must have 
\begin{align*}
    &\avg |\langle \mathbf{v}+(4z^{\star}+1)\mathbf{r},\beta^i \rangle|^2 \\
&    =\avg |\langle \mathbf{v},\beta^i \rangle|^2+(4z^{\star}+1)^{2}\avg |\langle \mathbf{r},\beta^i \rangle|^2
\\    
    &=1+(4z^{\star}+1)\sum_{i=1}^{2z^{\star}} i^2 \\
    &=1+\frac{z^{\star}(2z^{\star}+1)(4z^{\star}+1)^{2}}{3}.
\end{align*}
since $||\beta^i||_2=1$. Since the variance of the noise $\avg \eta^2$ is $\sigma^2$, we must have that 
\begin{align*}
    \mathsf{SNR}=\frac{1}{\sigma^2}\Big(1+\frac{z^{\star}(2z^{\star}+1)(4z^{\star}+1)^{2}}{3}\Big).
\end{align*}
Substituting the above expression in the statement of Theorem \ref{thm:noisygen} and using the fact that $z^{\star}$ is a constant, we get the statement of the corollary.
\end{proof}

\section{Proof of Lemma \ref{lem:RIPunion}}
 First, let us introduce a few notations. For a given any set of indices $\mathbf{T}\subset [n]$, denote by $\mathbf{X_T}$ the set of all vectors in $\mathbb{R}^{n}$ that are zero outside of $T$. We start by stating the Johnson-Linderstrauss Lemma proved in \cite{baraniuk2006johnson}.
\begin{lem}{[Lemma 5.1 in \cite{baraniuk2006johnson}]}\label{lem:JL}
Let $\mathbf{A}$ be a $m \times n$ matrix such that every element in $\mathbf{A}$ is sampled independently and uniformly at random from $\{1/\sqrt{m},-1/\sqrt{m}\}$. For any set $T \subset [n]$ such that $|T|=k$ and any $0<\delta<1$, we have
\begin{align*}
    (1-\delta)||\mathbf{x}||_{2} \le ||\mathbf{A}\mathbf{x}||_2 \le (1+\delta)||\mathbf{x}||_{2} \quad \textup{for all } \mathbf{x}\in \mathbf{X_T}
\end{align*}
with probability at least $1-2(12/\delta)^k e^{-\frac{m}{2}(\delta^2/8-\delta^3/24)}$. 
\end{lem}
We are now ready to prove Lemma \ref{lem:RIPunion}. Since there are $\binom{n}{k}$ distinct subsets of $[n]$ that are of size $k$, we take a union bound over all the subsets and therefore the failure probability of Lemma \ref{lem:JL} for all sets of indices of size $k$ (definition of $\delta$-$\mathsf{RIP}$) is at most 
\begin{align*}
2(12/\delta)^k\binom{n}{k}e^{-\frac{m}{2}(\delta^2/8-\delta^3/24)}.
\end{align*}
We need that from $\alpha m (\alpha>1)$ vectors randomly sampled from $\{\frac{1}{\sqrt{m}},\frac{-1}{\sqrt{m}}\}^{n}$ any $m$ vectors satisfy the $\delta$-$\mathsf{RIP}$ property for some value of $m$. Therefore, the probability of failure is at most
\begin{align*}
2\left(\frac{12}{\delta}\right)^k\binom{n}{k}\binom{\alpha m}{m}e^{-\frac{m}{2}(\delta^2/8-\delta^3/24)}.
\end{align*}
By Stirling's approximation and the fact that both $\alpha m$ and $m$ is large, we get that
\begin{align*}
  \binom{\alpha m}{m} \approx \sqrt{\frac{\alpha}{2\pi m(\alpha-1)}}\Big(\frac{\alpha^\alpha}{(\alpha-1)^{\alpha-1}}\Big)^{m}  
\end{align*}
Further we can also upper bound the binomial coefficients $\binom{n}{k}$ by $\Big(\frac{en}{k}\Big)^k$. Hence we can upper bound the failure probability as
\begin{align*}
    \exp\Big(-m(\delta^2/16-\delta^3/48)+m\log\Big(\frac{\alpha^\alpha}{(\alpha-1)^{\alpha-1}}\Big) +k\log(en/k)+\log(12/\delta)+\log 2\Big)
\end{align*}
Therefore, if we substitute $m=c'k\log (en/k)$ for some constant $c'>0$, we must have the failure probability to be upper bounded as $e^{-c''m(1+o(1))}$ for some $c''>0$ as long as we have
\begin{align*}
    c'(\frac{\delta^2}{16}-\frac{\delta^3}{48})>c'\log\Big(\frac{\alpha^\alpha}{(\alpha-1)^{\alpha-1}}\Big)+1
\end{align*}
implying that 
\begin{align*}
    \frac{\alpha^\alpha}{(\alpha-1)^{\alpha-1}} < \exp\Big(\frac{\delta^2}{16}-\frac{\delta^3}{48}-\frac{1}{c'}\Big). 
\end{align*}
Hence, by choosing the constant $c'$ appropriately large, the term in the exponent on the right hand side can be made positive. Since the left hand side of the equation is always greater than $1$, there will exist an $\alpha$ satisfying the equation.
\end{document}